\documentclass[a4paper,fleqn]{casdc}

\usepackage[numbers]{natbib}
\usepackage{graphicx}
\usepackage{algorithmicx}
\usepackage{algorithm}
\usepackage{multirow} 
\newtheorem{theorem}{Theorem}
\newenvironment{proof}{{\noindent\it Proof.}}{\hfill $\square$\par}

\usepackage{caption}

\def\tsc#1{\csdef{#1}{\textsc{\lowercase{#1}}\xspace}}
\tsc{WGM}
\tsc{QE}
\tsc{EP}
\tsc{PMS}
\tsc{BEC}
\tsc{DE}

\begin{document}
	\let\WriteBookmarks\relax
	\def\floatpagepagefraction{1}
	\def\textpagefraction{.001}
	\shorttitle{Constrained Bilinear Factorization Multi-view Subspace Clustering}
	\shortauthors{Qinghai Zheng et~al.}
	
	\title [mode = title]{Constrained Bilinear Factorization Multi-view Subspace Clustering}                      
	\author[mymainaddress]{Qinghai Zheng}[style=chinese]
	
	\author[mymainaddress]{Jihua Zhu}[style=chinese, orcid=0000-0002-3081-8781]
	\cormark[1]
	
	\author[mymainaddress]{Zhiqiang Tian}[style=chinese]
	
	\author[mymainaddress]{Zhongyu Li}[style=chinese]
	
	\author[mymainaddress]{Shanmin Pang}[style=chinese]
	
	\author[mysecondaddress]{Xiuyi Jia}[style=chinese]
	
	\address[mymainaddress]{School of Software Engineering, Xi'an Jiaotong University, Xian 710049, People's Republic of China}
	
	\address[mysecondaddress]{School of Computer Science and Engineering, Nanjing University of Science and Technology, Nanjing 210094, People's Republic of China}
	
	\cortext[cor1]{Corresponding author, email: zhujh@xjtu.edu.cn.}
	
	\begin{abstract}
		Multi-view clustering is an important and fundamental problem. Many multi-view subspace clustering methods have been proposed, and most of them assume that all views share a same coefficient matrix. However, the underlying information of multi-view data are not fully exploited under this assumption, since the coefficient matrices of different views should have the same clustering properties rather than be uniform among multiple views. To this end, this paper proposes a novel Constrained Bilinear Factorization Multi-view Subspace Clustering (CBF-MSC) method. Specifically, the bilinear factorization with an orthonormality constraint and a low-rank constraint is imposed for all coefficient matrices to make them have the same trace-norm instead of being equivalent, so as to explore the consensus information of multi-view data more fully. Finally, an Augmented Lagrangian Multiplier (ALM) based algorithm is designed to optimize the objective function. Comprehensive experiments tested on nine benchmark datasets validate the effectiveness and competitiveness of the proposed approach compared with several state-of-the-arts.
	\end{abstract}
	
	
	\maketitle
	
	\section{Introduction}
	\overfullrule=2cm
	Subspace clustering is an important technique in data mining and machine learning involved with many applications, such as face clustering and motion segmentation \cite{vidal2011SubspaceClustering}. It relies on the assumption that high-dimensional data points lie in a union of multiple low-dimensional subspaces and aims to group data points into corresponding clusters simultaneously \cite{elhamifar2013SparseSubspaceClustering}. Owing to its promising performance and good interpretability, a number of clustering algorithms based on subspace clustering have been proposed \cite{vidal2011SubspaceClustering,elhamifar2013SparseSubspaceClustering,liu2013LRR,wang2013LRSSC,hu2014SMoothRepresentation,chen2019localityKBS}. For example, Sparse Subspace Clustering (SSC) \cite{elhamifar2013SparseSubspaceClustering} obtains a sparsest subspace coefficient matrix for clustering. Besides, Low Rank Representation (LRR) \cite{liu2013LRR} finds a self-representation of the dataset under the low-rank constraint. Low Rank Sparse Subspace Clustering (LRSSC) \cite{wang2013LRSSC} employs both the low-rank constraint and the sparsity constraint on the self-representation simultaneously. SMooth Representation (SMR) \cite{hu2014SMoothRepresentation} investigates the grouping effect for subspace clustering. All these subspace clustering approaches have achieved promising clustering results in practice. However, they are proposed for the single-view data rather than the multi-view data \cite{xu2013ASurveyonMultiviewLearning}, which are widespread in many real-world applications.
	
	Multi-view data, collected from multiple sources or different measurements, are common in real-world applications \cite{zhao2017MultiviewLearningOverview}. For instance, images can be described by variant features (SIFT \cite{lowe2004SIFT}, LBP \cite{ojala2002LBP}, etc.); visual frames and audio signals are two distinct views of a video and both are important for multimedia content understanding. Compared with single-view data, multi-view data contain more useful information for learning and data mining \cite{xu2013ASurveyonMultiviewLearning,zhao2017MultiviewLearningOverview,chao2017ASurveyonMultiviewClustering,xie2018on,wang2019studyKBS,KNOS_Recom_2,KNOS_Recom_3,KNOS_Recom_4}. It is of vital importance to achieve the agreement or consensus information among multiple views during clustering. Obviously, running a single-view clustering algorithm on the multi-view data directly is not a good choice for multi-view clustering \cite{xu2013ASurveyonMultiviewLearning,chao2017ASurveyonMultiviewClustering,kumar2011Co-TrainingMultiviewSpectralClustering,kumar2011Co-regMultiviewSpectralClustering}. A number of multi-view subspace clustering approaches have been proposed in recent years \cite{gao2015MVSC,IJCAI16AMGLnie2016parameter,AAAI2018CSMSC,PR2018MultiviewLRSSC,PR2019ISGMC}. Although good clustering results have been achieved in practice, most existing multi-view subspace clustering approaches assume that all views have a same coefficient matrix to explore the consensus information. The above assumption is not proper for multi-view clustering, since different views have specific self-expressiveness properties. A more suitable way is to assume that the coefficient matrices of multiple views have the same underlying data distribution and clustering properties, instead of being uniform among all views.
	
	To address the above problem and explore the underlying information of multi-view data, we propose a novel Constrained Bilinear Factorization Multi-view Subspace Clustering, dubbed CBF-MSC, in this paper. By introducing the bilinear factorization \cite{cabral2013ICCV2013UnifyingBilinearFactorization} with an orthonormality constraint, the coefficient matrices of all views gain the same clustering properties and the consensus information of multi-view data can be well explored. Meanwhile, the specific information of different views is taken into consideration during clustering as well. Finally, an alternating direction minimization algorithm based on the Augmented Lagrangian Multiplier (ALM) \cite{lin2011ALM} is designed to optimize the objective function, and experimental results conducted on nine real-world datasets demonstrate its superiority over several state-of-the-art approaches for multi-view clustering.
	
	The main contributions of this work are delivered as follows:
	\begin{itemize}
		\item [1)] A novel Constrained Bilinear Factorization Multi-view Subspace Clustering (CBF-MSC) is proposed in this paper, which assumes that the coefficient matrices of multiple views have the same clustering properties and consensus information of multi-view data are explored fully.
		\item [2)] By introducing the constrained bilinear factorization, a coefficient matrix can be factorized into a view-specific basis matrix and a common shared encoding matrix. Furthermore, we prove that the coefficient matrices of all views can obtain the same trace-norm during clustering.
		\item [3)] An effective optimization algorithm is developed and extensive experiments are conducted on nine benchmark datasets so as to demonstrate the effectiveness and competitiveness of the proposed method for multi-view clustering. 
	\end{itemize}
	
	The remainder of this paper is organized as follows. Related works are briefly reviewed in Section \ref{Related Work}. Section \ref{Proposed Methodology} presents the proposed method in detail, optimization of which is developed in Section \ref{Optimization}. Experimental results, including convergence properties analysis and parameters sensitivity analysis, are provided in Section \ref{Experiments}. Finally, Section \ref{Conclusion} concludes the paper with a brief summary.
	
	\section{Related Work}
	\label{Related Work}
	Generally, multi-view clustering approaches \cite{KNOS_Recom_1} can be categorized into two main groups roughly: generative methods and discriminative methods \cite{chao2017ASurveyonMultiviewClustering}. Methods in the first category try to construct generative models for multiple views by learning the fundamental distribution of data. For instance, the work proposed in \cite{bickel2004MultiviewClustering-ICDM2004} explores the multi-view clustering problem under an assumption that views are dependent and multi-nomial distribution is applied for models construction; the multi-view clustering method in \cite{tzortzis2009CMM-MultiviewClustering} learns convex mixture models for multiple views. 
	
	Discriminative methods aim to simultaneously minimize the intra-cluster disimilarity and the inter-cluster similarity. Most existing multi-view clustering approaches belong to this category \cite{chao2017ASurveyonMultiviewClustering,zhang2017LMSC}, including spectral clustering based methods~\cite{kumar2011Co-TrainingMultiviewSpectralClustering,kumar2011Co-regMultiviewSpectralClustering,xia2014RMSC,zhu2018OneStepMultiviewSpectralClustering,KNOS_Recom_5,GMCTKDE2019} and subspace clustering based methods~\cite{gao2015MVSC,zhang2015LT-MSC,cao2015DiMSC,zhang2017LMSC,AAAI2018CSMSC,FCMSC,zhou2019dual}. For example, co-training multi-view spectral clustering~\cite{kumar2011Co-TrainingMultiviewSpectralClustering} leverages the eigenvectors of a graph Laplacian from one view to constrain other views; co-regularized multi-view spectral clustering~\cite{kumar2011Co-regMultiviewSpectralClustering} combines similarity matrices of multiple views to achieve clustering results by co-regularizing the clustering hypotheses among views; Robust Multi-view Spectral Clustering (RMSC)~\cite{xia2014RMSC} recovers a transition probability matrix via low-rank and sparse decomposition for clustering. {\color{black}{Multi-View Spectral Clustering via Integrating Global and Local Graphs \cite{KNOS_Recom_5} is a method which tries to minimize the partial sum of singular values of the transition probability matrix with a priori rank information, and a global graph from the concatenated features is also constructed to exploit the complementary information embedded in different views.}} Besides, many subspace clustering based methods are proposed in recent year as well. Low-rank Tensor constrained Multi-view Subspace Clustering (LT-MSC)~\cite{zhang2015LT-MSC} explores the high order correlations underlying multi-view data and the tensor with a low-rank constraint is employed. Diversity-induced Multi-view Subspace Clustering (DiMSC)~\cite{cao2015DiMSC} obtains clustering performance with the help of the diversity constraint. Latent Multi-view Subspace Clustering (LMSC)~\cite{zhang2017LMSC} achieves the promising clustering performance by exploring the underlying complementary information and seeking a latent representation of multi-view data at the same time. The propose CBF-MSC belongs to the second category, i.e., a discriminative method based on subspace clustering.
	
	The matrix factorization is often used for matrix approximation~\cite{cabral2013ICCV2013UnifyingBilinearFactorization,zheng2012PracticalLowRankMatrixApproximation-L1norm}. Our method takes advantage of the bilinear factorization to construct the coefficient matrices which contain the underlying clustering information of all views. It is noteworthy that our method is different from approaches proposed in~\cite{liu2013MultiviewClustering-JointNonnegativeMatrixFactorization,zhao2017MultiviewClustering-DeepMatrixFactorization,akata2011NMFinMultimodalityData}. \cite{liu2013MultiviewClustering-JointNonnegativeMatrixFactorization} formulates a joint nonnegative matrix factorization to keep the clustering results among multiple views comparable; \cite{zhao2017MultiviewClustering-DeepMatrixFactorization} constructs a deep matrix factorization framework to explore the consensus information by seeking for a common representation of multi-view data. These methods perform the nonnegative matrix factorization~\cite{lee1999NMF-Nature} on the data matrix, columns of which are data points, directly. However, since statistic properties of different views are diverse, it is risky to employ the matrix factorization on the data matrices straightforward to seek a common representation matrix for multi-view data. As for the proposed CBF-MSC, it imposes the constrained bilinear factorization on subspace coefficient matrices, which are self-representations with the same or similar clustering properties \cite{elhamifar2013SparseSubspaceClustering,liu2013LRR,chao2017ASurveyonMultiviewClustering,zhang2015LT-MSC}.
	
	\section{Proposed Methodology}
	\label{Proposed Methodology}
	In this section, we introduce the proposed CBF-MSC in detail. For convenience, Table \ref{table_symbols} lists main symbols applied in this paper and Fig. \ref{figure_Illustration} illustrates the proposed method. 
	
	\begin{table}
		\caption{Main symbols utilized in this paper.}
		\begin{center}
			\begin{tabular}{|l|l|}
				\hline
				Symbol & Meaning \\
				\hline\hline
				$n$ & The number of samples. \\
				$v$ & The number of views. \\
				$c$ & The number of clusters.\\
				$d_i$ & The dimension of the $i$-th view.\\
				${x_j^i} \in {R^{d_i}}$ & The $j$-th data point from the $i$-th view.\\
				${X^{(i)}} \in {R^{{d_i} \times n}}$ & The data matrix of the $i$-th view.\\
				${\left\| A \right\|_{2,1}}$ & The $l_{2,1}$ norm of matrix A.\\
				${\left\| A \right\|_ * }$ & The trace norm of matrix A.\\
				${\left\| A \right\|_ F^2 }$ & The Frobenius norm of matrix A.\\
				${\mathop{\rm rank}\nolimits}(A)$ & The rank of matrix A.\\
				${\mathop{\rm abs}\nolimits}(\cdot)$ & The element-wise absolute operator.\\
				\hline
			\end{tabular}
		\end{center}
		\label{table_symbols}
	\end{table}
	
	\begin{figure*}
		\begin{center}
			\includegraphics[width=0.85\linewidth]{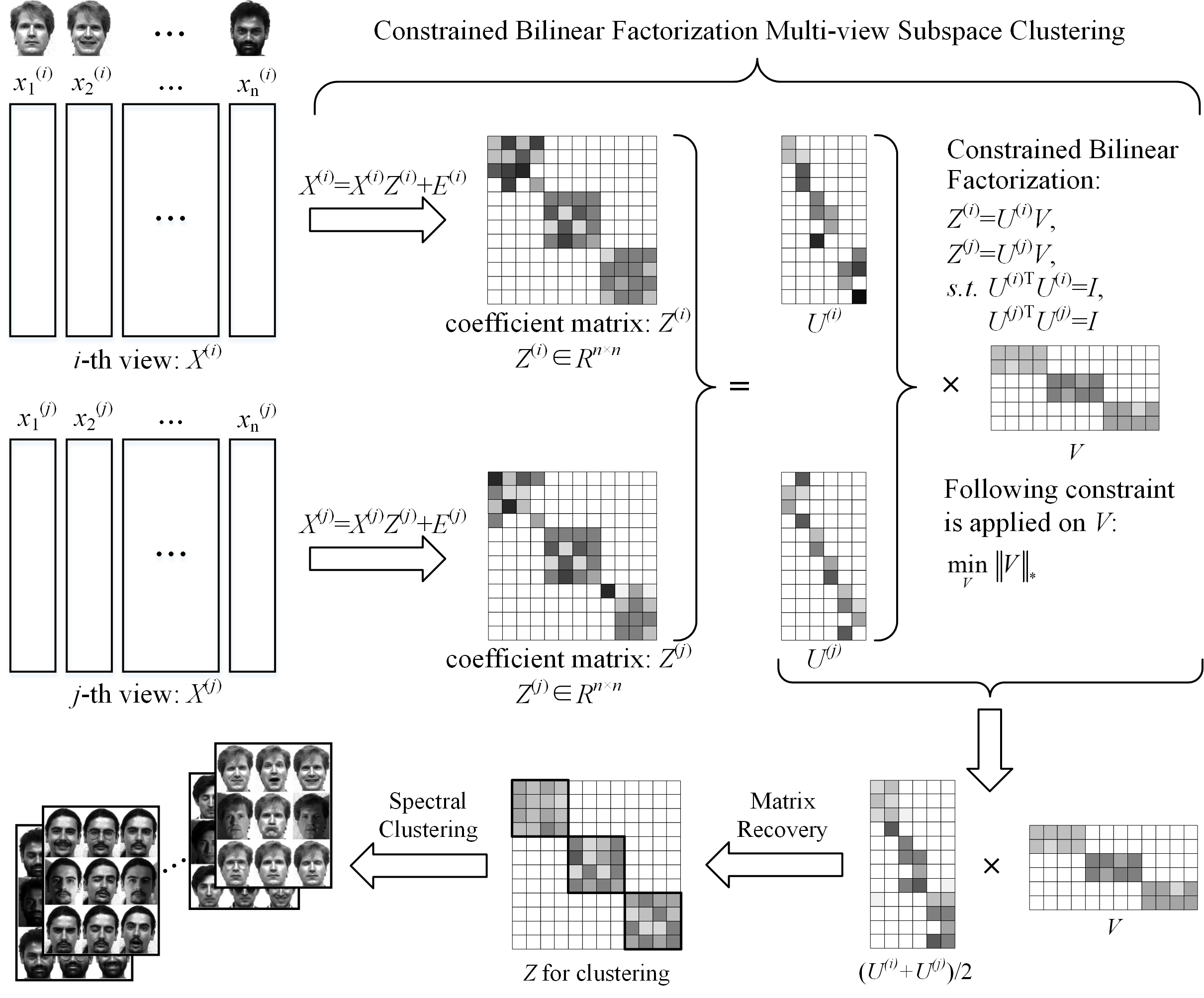}
		\end{center}
		\caption{Overview of the proposed CBF-MSC. Taking the $i$-th view and the $j$-th view here for example, clustering results of the $i$-th view and the $j$-th view are same and it is clear that $Z^{(i)}$ is not the same with $Z^{(j)}$. Actually, our method assumes that $\rm{rank(}Z^{(i)}\rm{)}$ equals to $\rm{rank(}Z^{(j)}\rm{)}$, and to this end, the constrained bilinear factorization is introduced in the proposed method. Finally, the multi-view clustering results of the proposed CBF-MSC are inferred by leveraging the spectral clustering with the adjacency matrix 
			$({\mathop{\rm abs}\nolimits} (Z) + {\mathop{\rm abs}\nolimits} ({Z^T}))/2$.}
		\label{figure_Illustration}
	\end{figure*}
	
	Subspace clustering based on self-representation is a fundamental and important technique~\cite{vidal2011SubspaceClustering,zhang2017LMSC}. Taking Low Rank Representation (LRR)~\cite{liu2013LRR} for example, given the $i$-th view  $X^{(i)}$ with $n$ samples and $d_i$ dimension, data points of which are drawn from $c$ subspaces, the objective function of subspace clustering can be written as follows:
	\begin{equation}
	\mathop {\min }\limits_{Z^{(i)}} {\kern 1pt} {\kern 1pt} {\kern 1pt} {\left\| E^{(i)} \right\|_{2,1}} + \lambda \left\| Z^{(i)} \right\|_*{\kern 1pt} {\kern 1pt} {\kern 1pt} {\kern 1pt} {\rm{s}}{\rm{.t}}{\rm{.}}{\kern 1pt} {\kern 1pt} {\kern 1pt} X^{(i)} = X^{(i)}Z^{(i)} + E^{(i)},
	\label{LRR_singleview}
	\end{equation}
	where $\lambda$ is the tradeoff parameter, $Z^{(i)} \in {R^{n \times n}}$ denotes the subspace coefficient matrix of $X^{(i)}$, and $E^{(i)}\in {R^{d \times n}}$ indicates the sample-specific error. 
	
	\subsection{Formulation}
	Given a multi-view dataset with $v$ views, i.e. $\{ {X^{(i)}}\} _{i = 1}^v$. It is clear that $Z^{(i)}$ in Equation (\ref{LRR_singleview}) is vital for clustering and the most existing multi-view subspace clustering methods pursuit a shared coefficient matrix for all views. However, in real-world applications, data points from different views have the same clustering results, rather than the same coefficient matrix. In other words, it is not a proper way to obtain the multi-view clustering results with a same coefficient matrix. In this paper, CBF-MSC assumes that the coefficient matrices have a same rank among multiple views, since the rank of a coefficient matrix is crucial for clustering \cite{liu2013LRR}.
	\begin{equation}
	\begin{array}{l}
	\mathop {\min }\limits_{{E^{(i)}},{Z^{(i)}}} \sum\limits_{i = 1, \cdots ,v} {{{\left\| {{E^{(i)}}} \right\|}_{2,1}} + \lambda {\mathop{\rm rank}\nolimits} ({Z^{(i)}})} \\
	{\rm{s}}{\rm{.t}}{\rm{.}}{\kern 1pt} {\kern 1pt} {\kern 1pt} {X^{(i)}} = {X^{(i)}}{Z^{(i)}} + {E^{(i)}}\\
	{\kern 1pt} {\kern 1pt} {\kern 1pt} {\kern 1pt} {\kern 1pt} {\kern 1pt} {\kern 1pt} {\kern 1pt} {\kern 1pt} {\kern 1pt} {\kern 1pt} {\kern 1pt} {\kern 1pt} {\kern 1pt} {\kern 1pt} {\kern 1pt} {\mathop{\rm rank}\nolimits} ({Z^{(1)}}) = {\kern 1pt} {\mathop{\rm rank}\nolimits} ({Z^{(2)}}) =  \cdots  = {\kern 1pt} {\mathop{\rm rank}\nolimits} ({Z^{(v)}}).
	\end{array}
	\end{equation}
	
	Due to the discrete nature of the rank function, it is hard to optimize the above problem \cite{liu2013LRR}. Since the trace norm regularization promotes low rank solutions and the trace norm of $Z^{(i)}$ is equal to the $l_1$-norm on the singular values of $Z^{(i)}$ \cite{fazel2001a}, the following problem can be constructed straightforward:
	\begin{equation}
	\begin{array}{l}
	\mathop {\min }\limits_{{E^{(i)}},{Z^{(i)}}} \sum\limits_{i = 1, \cdots ,v} {{{\left\| {{E^{(i)}}} \right\|}_{2,1}} + \lambda {{\left\| {{Z^{(i)}}} \right\|}_ * }} \\
	{\rm{s}}{\rm{.t}}{\rm{.}}{\kern 1pt} {\kern 1pt} {\kern 1pt} {X^{(i)}} = {X^{(i)}}{Z^{(i)}} + {E^{(i)}}\\
	{\kern 1pt} {\kern 1pt} {\kern 1pt} {\kern 1pt} {\kern 1pt} {\kern 1pt} {\kern 1pt} {\kern 1pt} {\kern 1pt} {\kern 1pt} {\kern 1pt} {\kern 1pt} {\kern 1pt} {\kern 1pt} {\kern 1pt} {\kern 1pt} {\left\| {{Z^{(1)}}} \right\|_ * } = {\kern 1pt} {\left\| {{Z^{(2)}}} \right\|_ * } =  \cdots  = {\kern 1pt} {\left\| {{Z^{(v)}}} \right\|_ * }.
	\end{array}
	\label{CBFMSC_1}
	\end{equation}
	
	The Equation (\ref{CBFMSC_1}) is formulated under a suitable assumption for multi-view clustering, however, it is difficult to be well optimized due to the same rank constraint. Therefore, the further improvement should be employed. As shown in Fig. \ref{figure_Illustration}, the constrained bilinear factorization is introduced in our CBF-MSC and the coefficient matrix of the $i$-th view can be written as follows:
	\begin{equation}
	{Z^{(i)}} = {U^{(i)}}V{\kern 1pt} {\kern 1pt} {\kern 1pt} {\kern 1pt} {\kern 1pt} {\rm{s}}{\rm{.t}}{\rm{.}}{\kern 1pt} {\kern 1pt} {\kern 1pt} {U^{{{(i)}^T}}}{U^{(i)}} = I,
	\label{BilinearFactorization_Introducted}
	\end{equation}
	where $I$ denotes an identity matrix with proper size, ${U^{(i)}} \in {R^{n \times k}}$ indicates the basis matrix of the matrix factorization with respect to $Z^{(i)}$, and $V \in {R^{k \times n}}$ is an encoding matrix of $Z^{(i)}$ based on $U^{(i)}$, $ k \ge c$. Since each view has the specific information, the subspace coefficient matrices of multiple views are different from each other to some degree, and it is reasonable that the basis matrix of the bilinear factorization varies with different views. Moreover, it is worth noting that the orthonormality constraint of $U^{(i)}$ plays an significant role here. To be specific, we have the following theorem.
	\begin{theorem}
		\label{Theorem1}
		Given a square matrix $S$, which can be factorized as $S = LR$, if the constraint $L^TL=I$ holds, the trace norm of $S$ equals to the trace norm of $R$:
		\begin{equation}
		{\left\| S \right\|_*} = {\left\| R \right\|_*}.
		\end{equation}
	\end{theorem}
	
	\begin{proof}
		Since ${S^T}S = {(LR)^T}(LR) = {R^T}{L^T}LR = {R^T}R$, the matrix of ${S^T}S$ has the same eigenvalues with the matrix of ${R^T}R$. Accordingly, the matrix $S$ and the matrix $R$ have the same singular values, namely, the trace norm of $S$ equals to the trace norm of $R$, i.e. ${\left\| S \right\|_*} = {\left\| R \right\|_*}$.
	\end{proof}
	
	For an arbitrary view of multi-view data $X^{(i)}$, according to Theorem \ref{Theorem1}, the trace norm of the subspace coefficient matrix equals to the trace norm of $V$ subjected to the orthonormality constraint of $U^{(i)}$. That is to say, by performing the constrained bilinear factorization on the low rank representation of multiple views, the following formula can be constructed directly:
	\begin{equation}
	\begin{array}{l}
	{\left\| {{Z^{(1)}}} \right\|_*} = {\left\| {{Z^{(2)}}} \right\|_*} =  \cdots  = {\left\| {{Z^{(v)}}} \right\|_*} = {\left\| V \right\|_*}\\
	{\rm{s}}{\rm{.t}}{\rm{.}}{\kern 1pt} {\kern 1pt} {\kern 1pt} {Z^{(i)}} = {U^{(i)}}V,{\kern 1pt} {\kern 1pt} {\kern 1pt} {\kern 1pt} {\kern 1pt} {\kern 1pt} {U^{{{(i)}^T}}}{U^{(i)}} = I,
	\end{array}
	\end{equation}
	where ${\kern 1pt} i = 1,2, \cdots ,v$ and it also indicates that $V$, which is invariant with different views, is a proper choice. Consequently, The problem of the Equation (\ref{CBFMSC_1}) is equivalent to the following problem:
	\begin{equation}
	\begin{array}{l}
	\mathop {\min }\limits_{{E^{(i)}},{Z^{(i)}}} \sum\limits_{i = 1, \cdots ,v} {{{\left\| {{E^{(i)}}} \right\|}_{2,1}}}  + \lambda {\left\| V \right\|_ * }\\
	{\rm{s}}{\rm{.t}}{\rm{.}}{\kern 1pt} {\kern 1pt} {\kern 1pt} {X^{(i)}} = {X^{(i)}}{Z^{(i)}} + {E^{(i)}}\\
	{\kern 1pt} {\kern 1pt} {\kern 1pt} {\kern 1pt} {\kern 1pt} {\kern 1pt} {\kern 1pt} {\kern 1pt} {\kern 1pt} {\kern 1pt} {\kern 1pt} {\kern 1pt} {\kern 1pt} {\kern 1pt} {\kern 1pt} {\kern 1pt} {Z^{(i)}} = {U^{(i)}}V\\
	{\kern 1pt} {\kern 1pt} {\kern 1pt} {\kern 1pt} {\kern 1pt} {\kern 1pt} {\kern 1pt} {\kern 1pt} {\kern 1pt} {\kern 1pt} {\kern 1pt} {\kern 1pt} {\kern 1pt} {\kern 1pt} {\kern 1pt} {\kern 1pt} {U^{{{(i)}^T}}}{U^{(i)}} = I
	\end{array}
	\label{CBFMSC_ObjectiveFunction}
	\end{equation}
	where $\lambda$ is the tradeoff parameter, $V$ denotes the common shared encoding matrix which contains the consensus information, and $\{ {U^{(i)}}\} _{i = 1}^v$ indicates the view-specific information, therefore, the complementary information of multi-view data is also concerned during clustering. 
	
	Once $V$ and $\{ {U^{(i)}}\} _{i = 1}^v$ are optimized, a subspace coefficient matrix $Z$, which contains the consensus information of multiple views, can be achieved as follows:
	\begin{equation}
	Z = \frac{1}{v}\sum\limits_{i = 1}^v {{U^{(i)}}V},
	\end{equation}
	
	\section{Optimization}
	\label{Optimization}
	In this section, an algorithm is designed for optimizing the objective function of the proposed CBF-MSC effectively. Besides analyses of the computational complexity and convergence are presented as well.
	
	\subsection{Optimization of CBF-MSC}
	To solve the  objective function of Equation (\ref{CBFMSC_ObjectiveFunction}), an algorithm based on the Augmented Lagrangian Multiplier (ALM) \cite{lin2011ALM} method is developed in this section. Besides, the alternating  direction  minimization  strategy is employed and an auxiliary variable $L$ is introduced here to make the optimization procedure separable. Therefore, the objective function can be rewritten as follows:
	\begin{equation}
	\begin{array}{l}
	\mathop {\min }\limits_{{E^{(i)}},{Z^{(i)}}} \sum\limits_{i = 1, \cdots ,v} {{{\left\| {{E^{(i)}}} \right\|}_{2,1}}}  + \lambda {\left\| L \right\|_ * }\\
	{\rm{s}}{\rm{.t}}{\rm{.}}{\kern 1pt} {\kern 1pt} {\kern 1pt} {X^{(i)}} = {X^{(i)}}{Z^{(i)}} + {E^{(i)}}\\
	{\kern 1pt} {\kern 1pt} {\kern 1pt} {\kern 1pt} {\kern 1pt} {\kern 1pt} {\kern 1pt} {\kern 1pt} {\kern 1pt} {\kern 1pt} {\kern 1pt} {\kern 1pt} {\kern 1pt} {\kern 1pt} {\kern 1pt} {Z^{(i)}} = {U^{(i)}}V\\
	{\kern 1pt} {\kern 1pt} {\kern 1pt} {\kern 1pt} {\kern 1pt} {\kern 1pt} {\kern 1pt} {\kern 1pt} {\kern 1pt} {\kern 1pt} {\kern 1pt} {\kern 1pt} {\kern 1pt} {\kern 1pt} {\kern 1pt} {\kern 1pt} V = L\\
	{\kern 1pt} {\kern 1pt} {\kern 1pt} {\kern 1pt} {\kern 1pt} {\kern 1pt} {\kern 1pt} {\kern 1pt} {\kern 1pt} {\kern 1pt} {\kern 1pt} {\kern 1pt} {\kern 1pt} {\kern 1pt} {\kern 1pt} {\kern 1pt} {U^{{{(i)}^T}}}{U^{(i)}} = I.
	\end{array}
	\end{equation}
	The corresponding augmented Lagrange function is:
	\begin{equation}
	\begin{array}{l}
	{\cal L}(V,L,{U^{(i)}},{Z^{(i)}},{E^{(i)}})\\
	{\kern 1pt} {\kern 1pt} {\kern 1pt} {\kern 1pt} {\kern 1pt} {\kern 1pt} {\kern 1pt} {\kern 1pt} {\kern 1pt} {\kern 1pt} {\kern 1pt} {\kern 1pt} {\rm{ = }}\sum\limits_{i = 1}^v {{{\left\| {{E^{(i)}}} \right\|}_{2,1}}}  + \lambda {\left\| L \right\|_ * }\\
	{\kern 1pt} {\kern 1pt} {\kern 1pt} {\kern 1pt} {\kern 1pt} {\kern 1pt} {\kern 1pt} {\kern 1pt} {\kern 1pt} {\kern 1pt} {\kern 1pt} {\kern 1pt}  + \sum\limits_{i = 1}^v {\Gamma \left( {Y_1^{(i)},{X^{(i)}} - {X^{(i)}}{Z^{(i)}} - {E^{(i)}}} \right)} \\
	{\kern 1pt} {\kern 1pt} {\kern 1pt} {\kern 1pt} {\kern 1pt} {\kern 1pt} {\kern 1pt} {\kern 1pt} {\kern 1pt} {\kern 1pt} {\kern 1pt} {\kern 1pt}  + \sum\limits_{i = 1}^v {\Gamma \left( {Y_2^{(i)},{Z^{(i)}} - {U^{(i)}}V} \right)}  + \Gamma \left( {{Y_3},V - L} \right),
	\end{array}
	\label{CBFMSC_ALM_Function}
	\end{equation}
	where $Y_1^{(i)}$, $Y_2^{(i)}$ and $Y_3$ are Lagrange multipliers, and for the concise representation, $\Gamma (A,B)$ is defined as follows:
	\begin{equation}
	\Gamma (A,B) = \left\langle {A,B} \right\rangle  + \frac{\mu }{2}\left\| B \right\|_F^2,
	\end{equation}
	in which $\mu$ indicates a positive adaptive penalty parameter, $\left\langle { \cdot , \cdot } \right\rangle$ denotes the inner product of two matrices.
	
	The optimization of Equation (\ref{CBFMSC_ObjectiveFunction}) can be solved effectively by minimizing the corresponding ALM Equation (\ref{CBFMSC_ALM_Function}), which can be decomposed into six subproblems.
	
	\textbf{1) Subproblem of updating $E^{(i)}$:} The subproblem of updating $E^{(i)}$ can be written as follows when other variables are fixed:
	\begin{equation}
	\mathop {\min }\limits_{{E^{(i)}}} \sum\limits_{i = 1}^v {\left( {{{\left\| {{E^{(i)}}} \right\|}_{2,1}}{\rm{ + }}\Gamma \left( {Y_1^{(i)},{X^{(i)}} - {X^{(i)}}{Z^{(i)}} - {E^{(i)}}} \right)} \right)},
	\end{equation}
	results of which are the same with the following problem:
	\begin{equation}
	\mathop {\min }\limits_{{E^{(i)}}} \sum\limits_{i = 1}^v {\left( {{{\left\| {{E^{(i)}}} \right\|}_{2,1}}{\rm{ + }}\frac{\mu }{2}\left\| {{E^{(i)}} - T_E^{(i)}} \right\|_F^2} \right)},
	\label{E_subproblem}
	\end{equation}
	in which $T_E^{(i)}$ is defined as follows:
	\begin{equation}
	T_E^{(i)} = \left( {{X^{(i)}} - {X^{(i)}}{Z^{(i)}} + \frac{{Y_1^{(i)}}}{\mu }} \right),
	\end{equation}
	and according to Lemma 4.1 in \cite{liu2013LRR}, Equation (\ref{E_subproblem}) can be optimized effectively with a closed form.
	
	\textbf{2) Subproblem of updating $U^{(i)}$:} With other variables being fixed, the subproblem with respect to $U^{(i)}$ is:
	\begin{equation}
	\mathop {\min }\limits_{{U^{(i)}}} \sum\limits_{i = 1}^v {\Gamma \left( {Y_2^{(i)},{Z^{(i)}} - {U^{(i)}}V} \right)} {\kern 1pt} {\kern 1pt} {\kern 1pt} {\kern 1pt} {\kern 1pt} {\kern 1pt} {\rm{s}}{\rm{.t}}{\rm{.}}{\kern 1pt} {\kern 1pt} {\kern 1pt} {U^{{{(i)}^T}}}{U^{(i)}} = I,
	\end{equation}
	which can be reformulated as follows:
	\begin{equation}
	\begin{array}{l}
	\mathop {\min }\limits_{{U^{(i)}}} \sum\limits_{i = 1}^v {\frac{\mu }{2}\left\| {{{\left( {{Z^{(i)}} + \frac{{Y_2^{(i)}}}{\mu }} \right)}^T} - {V^T}{U^{{{(i)}^T}}}} \right\|_F^2} {\kern 1pt} {\kern 1pt} {\kern 1pt} {\kern 1pt} {\kern 1pt} {\kern 1pt} \\
	{\rm{s}}{\rm{.t}}{\rm{.}}{\kern 1pt} {\kern 1pt} {\kern 1pt} {U^{{{(i)}^T}}}{U^{(i)}} = I
	\end{array}
	\label{U_subproblem}
	\end{equation}
	which is a typical orthogonal procrustes problem~\cite{Peter1966OPPproposed}, and can be solved effectively according to \cite{OPPsolution}.
	
	\textbf{3) Subproblem of updating $V$:} To update $V$ with other variables fixed, we solve the following problem:
	\begin{equation}
	\mathop {\min }\limits_V \sum\limits_{i = 1}^v {\Gamma \left( {Y_2^{(i)},{Z^{(i)}} - {U^{(i)}}V} \right)}  + \Gamma \left( {{Y_3},V - L} \right),
	\end{equation}
	and we obtain the following optimization by taking the derivative with respect to $V$ and letting to be 0,
	\begin{equation}
	V = {T_{VA}}^{ - 1}{T_{VB}},
	\end{equation}
	with
	\begin{equation}
	\begin{array}{l}
	{T_{VA}} = \mu \left( {I + \sum\limits_{i = 1}^v {{U^{{{(i)}^T}}}{U^{(i)}}} } \right),\\
	{T_{VB}} = \mu L - {Y_3} + \sum\limits_{i = 1}^v {\left( {{U^{{{(i)}^T}}}Y_2^{(i)} + \mu {U^{{{(i)}^T}}}{Z^{(i)}}} \right).} 
	\end{array}
	\end{equation}
	
	\textbf{4) Subproblem of updating $L$:} The auxiliary variable Q can be updated as follows:
	\begin{equation}
	\mathop {\min }\limits_L \lambda {\left\| L \right\|_ * } + \frac{\mu }{2}\left\| {L - \left( {V + \frac{{{Y_3}}}{\mu }} \right)} \right\|_F^2,
	\label{L_subproblem}
	\end{equation}
	and to obtain the solution of Equation (\ref{L_subproblem}), a soft-threshold operator \cite{lin2011ALM} ${S_\varepsilon }$ is defined as follows:
	\begin{equation}
	{S_\varepsilon }(x) = \left\{ \begin{array}{l}
	x - \varepsilon ,{\kern 1pt} {\kern 1pt} {\kern 1pt} {\rm{if}}{\kern 1pt} {\kern 1pt} {\kern 1pt} x - \varepsilon  > 0\\
	x + \varepsilon ,{\kern 1pt} {\kern 1pt} {\kern 1pt} {\rm{if}}{\kern 1pt} {\kern 1pt} {\kern 1pt} x - \varepsilon  < 0\\
	{\kern 1pt} {\kern 1pt} {\kern 1pt} {\kern 1pt} {\kern 1pt} {\kern 1pt} {\kern 1pt} {\kern 1pt} 0{\kern 1pt} {\kern 1pt} {\kern 1pt} {\kern 1pt} {\kern 1pt} {\kern 1pt} {\kern 1pt} {\kern 1pt} {\kern 1pt} {\kern 1pt} ,{\kern 1pt} {\kern 1pt} {\kern 1pt} {\rm{otherwise}}
	\end{array} \right.
	\end{equation}
	and by performing SVD on ${\left( {V + \frac{{{Y_3}}}{\mu }} \right)}$, i.e. ${\left( {V + \frac{{{Y_3}}}{\mu }} \right)} = {U_L}\Sigma V_L^T$, the optimization of $L$ is
	\begin{equation}
	L = {U_L}{S_{{\lambda _2}/\mu }}(\Sigma )V_L^T.
	\end{equation}
	
	\textbf{5) Subproblem of updating ${Z^{(i)}}$:} With other variables fixed, the subproblem of updating can be formulated as follows:
	\begin{equation}
	\begin{array}{l}
	\mathop {\min }\limits_{{Z^{(i)}}} \sum\limits_{i = 1}^v {\Gamma \left( {Y_1^{(i)},{X^{(i)}} - {X^{(i)}}{Z^{(i)}} - {E^{(i)}}} \right)} \\
	{\kern 1pt} {\kern 1pt} {\kern 1pt} {\kern 1pt} {\kern 1pt} {\kern 1pt} {\kern 1pt} {\kern 1pt} {\kern 1pt} {\kern 1pt} {\kern 1pt} {\kern 1pt} {\kern 1pt} {\kern 1pt} {\kern 1pt} {\kern 1pt} {\kern 1pt} {\kern 1pt} {\kern 1pt}  + \sum\limits_{i = 1}^v {\Gamma \left( {Y_2^{(i)},{Z^{(i)}} - {U^{(i)}}V} \right)} 
	\end{array},
	\end{equation}
	and we take the derivative of the above function with respect to ${Z^{(i)}}$, set it to be 0, then the optimization of ${Z^{(i)}}$ can be obtained as follows:
	\begin{equation}
	{Z^{(i)}} = T_{ZA}^{{{(i)}^{ - 1}}}T_{ZB}^{(i)},
	\end{equation}
	in which ${T_{ZA}^{(i)}}$ and ${T_{ZB}^{(i)}}$ can be written as follows:
	\begin{equation}
	\begin{array}{l}
	T_{ZA}^{(i)} = \mu \left( {I + {X^{{{(i)}^T}}}{X^{(i)}}} \right),\\
	T_{ZB}^{(i)} = \mu {U^{(i)}}V - Y_2^{(i)} + {X^{{{(i)}^T}}}\left( {Y_1^{(i)} + \mu {X^{(i)}} - \mu {E^{(i)}}} \right).
	\end{array}
	\end{equation}
	
	\textbf{6) Subproblem of updating multipliers and $\mu$:} The multipliers can be updated as follows according to~\cite{lin2011ALM}:
	\begin{equation}
	\left\{ \begin{array}{l}
	Y_1^{(i)} = Y_1^{(i)} + \mu ({X^{(i)}} - {X^{(i)}}{Z^{(i)}} - E^{(i)})\\
	Y_2^{(i)} = Y_2^{(i)} + \mu ({Z^{(i)}} - {U^{(i)}}V)\\
	{Y_3} = {Y_3} + \mu (V - L).
	\end{array} \right.
	\end{equation}
	Besides, we update the parameter $\mu$ in each iteration with a nonnegative scalar $\rho$ and a threshold value $\mu_{max}$, i.e. $\mu  = \min (\rho \mu ,{\mu _{\max }})$.
	
	Algorithm 1 presents the whole procedure of our proposed optimization.
	\begin{algorithm}
		\caption{Algorithm of CBF-MSC}
		{\bf Input:} \\
		\hspace*{0.12in} Multi-view $\{ {X^{(i)}}\} _{i = 1}^v$, $U^{(i)}=0$, $V=0$, $L=0$, $Y_1^{(i)}=0$, $Y_2^{(i)}=0$, $Y_3=0$,\\
		\hspace*{0.12in} $\rho=1.9$, $\mu  = {10^{ - 4}}$, ${\mu _{\max }} = {10^{6}}$, $\varepsilon  = {10^{ - 6}}$, $Z^{(i)}$ with random initialization;\\
		{\bf Output:} \\ 
		\hspace*{0.12in} $\{ {U^{(i)}}\} _{i = 1}^v$, $V$,
		
		{\bf Repeat:}\\
		\hspace*{0.12in} Updating $V$ according to the subproblem 3;\\
		\hspace*{0.12in} Updating $L$ according to the subproblem 4;\\
		\hspace*{0.12in} {\bf For} $i = 1,2 \cdots ,v$ {\bf do}:\\
		\hspace*{0.36in} Updating $E^{(i)}$ according to the subproblem 1;\\
		\hspace*{0.36in} Updating $U^{(i)}$ according to the subproblem 2;\\
		\hspace*{0.36in} Updating $Z^{(i)}$ according to the subproblem 5;\\
		\hspace*{0.36in} Updating $Y_1^{(i)}$ according to the subproblem 6;\\
		\hspace*{0.36in} Updating $Y_2^{(i)}$ according to the subproblem 6;\\
		\hspace*{0.12in} {\bf End}\\
		\hspace*{0.12in} Updating $Y_3$ and $\mu$ according to the subproblem 6;\\
		{\bf Until:}\\
		\hspace*{0.12in} {\bf For} $i = 1,2 \cdots ,v$ {\bf do}:\\
		\hspace*{0.36in} ${\left\| {{X^{(i)}} - {X^{(i)}}{Z^{(i)}} - E^{(i)}} \right\|_\infty } < \varepsilon, $\\
		\hspace*{0.36in} ${\left\| {{Z^{(i)}} - {U^{(i)}}V} \right\|_\infty } < \varepsilon, $\\
		\hspace*{0.12in} {\bf End}\\
		\hspace*{0.12in} and ${\left\| {V - L} \right\|_\infty } < \varepsilon. $
	\end{algorithm}
	
	\begin{table}
		\caption{{\color{black}{Statistics of nine datasets used in this paper.}}}\label{data_statistics}
		\begin{center}
			\begin{tabular}{l|c|c|c}
				\hline
				\bf {\color{black}{Dataset}} & \bf {\color{black}{Instance}} & \bf {\color{black}{Views}} & \bf {\color{black}{Clusters}}\\
				\hline \hline
				{\color{black}{MSRCV1}} & {\color{black}{210}} & {\color{black}{6}} & {\color{black}{7}} \\
				\hline
				{\color{black}{3-sources}} & {\color{black}{169}} & {\color{black}{3}} & {\color{black}{6}} \\
				\hline
				{\color{black}{BBC}} & {\color{black}{685}} & {\color{black}{4}} & {\color{black}{4}} \\
				\hline
				{\color{black}{NGs}} & {\color{black}{500}} & {\color{black}{3}} & {\color{black}{5}} \\
				\hline
				{\color{black}{Yale Face}} & {\color{black}{165}} & {\color{black}{3}} & {\color{black}{15}} \\
				\hline
				{\color{black}{Movie 617}} & {\color{black}{617}} & {\color{black}{2}} & {\color{black}{17}} \\
				\hline
				{\color{black}{BBCSport}} & {\color{black}{544}} & {\color{black}{2}} & {\color{black}{5}} \\
				\hline
				{\color{black}{UCI digits}} & {\color{black}{2000}} & {\color{black}{3}} & {\color{black}{10}} \\
				\hline
				{\color{black}{Caltech101}} & {\color{black}{441}} & {\color{black}{3}} & {\color{black}{7}} \\
				\hline
			\end{tabular}
		\end{center}
	\end{table}

	\subsection{Computational Complexity and Convergence}
	The main computational consists of five parts, i.e., the subproblems of 1-5. Staying the same with Table \ref{table_symbols}, $n$ is the number of samples, $v$ denotes the number of views, $d_i$ indicates the dimension of the $i$-th views and $k$ is a dimensional value introduced during factorization, as shown in the Equation (\ref{BilinearFactorization_Introducted}). According to Algorithm 1, the complexities of updating $V$ and $L$ are $O(kn^2+k^2n+k^3)$ and $O(n^3)$ respectively. As for the subproblem of updating $E^{(i)}$, the complexity is $O(d_in^2+n^3)$, and $O(k^2n+n^3)$ is the complexity of updating $U^{(i)}$. For updating $Z^{(i)}$, Sylvester equation is optimized, and the complexity is $O(d_i^3 +n^3)$.  To sum up, since $k$ and $v$ are much smaller than $d_i$ and $n$ in practice, the computational complexity of the proposed algorithm is $O(d_{max}^3+n^3)$, where ${d_{\max }} = \max (\{ {d_i}\} _{i = 1}^v)$.

	For the proposed method, it is difficult to prove its convergence, since more than two subproblems are involved during optimization. Inspired by \cite{yin2015dual,zhang2018discriminative,zhang2017discriminative}, the convergence discussion will be presented in the experiments section, and comprehensive results shown in next section illustrate the strong and stable convergence of the proposed algorithm.
	
	\section{Experiments}
	\label{Experiments}
	In this section, we demostrate the effectiveness of the porposed CBF-MSC. Experimental results and corresponding analysis are presented. All codes are implemented in Matlab on a desktop with a four-core 3.6GHz processor and 8GB of memory.
	\subsection{Experimental settings}
	Comprehensive experiments are conducted on nine real-world multi-view datasets to evaluate the proposed approach. To be specific, MSRCV1\footnote{http://research.microsoft.com/en-us/projects/objectclassrecognition/} consists of 210 image samples collected from 7 clusters with 6 views, including CENT, CMT, GIST, HOG, LBP, and SIFT. 3-sources\footnote{http://mlg.ucd.ie/datasets/3sources.html} is a news articles dataset which come from BBC, Reuters, and Guardian. BBC\footnote{http://mlg.ucd.ie/datasets/bbc.html} consists of 685 new documents from BBC and each of which is divided into 4 sub-parts. NGs\footnote{http://lig-membres.imag.fr/grimal/data.html} is a NewsGroups dataset consisting of 500 samples and has 3 views collected by 3 different methods. Yale Face\footnote{http://cvc.cs.yale.edu/cvc/projects/yalefaces/yalefaces.html} consists of 165 images from 15 individuals.  Movie 617\footnote{http://lig-membres.imag.fr/grimal/data/movies617.tar.gz}, is a movie dataset containing 617 movies of 17 genres, and consists of two views, including keywords and actors.
	{\color{black}{BBCSport \cite{xia2014RMSC} consists of 544 documents, each of which is divided into two sub-parts as two different views, and it is collected from the BBC Sport website corresponding to sports news in 5 topical areas, and the standard TF-IDF normalization is utilized to get the corresponding features. UCI digits \cite{asuncion2007uci} consist of 2000 images collected from 10 clusters with 3 views. Caltehch101 \cite{Caltech101} used in this section contains 441 instances collected from 7 clusters with 3 views.}}

	Clustering results of our method are compared with several baselines, as follows:
	
	1) ${\rm LRR_{BSV}}$ \cite{liu2013LRR} : Best Single-View clustering results based on Low-Rank Representation. LRR is employed on each single view, the best clustering performance based on the corresponding coefficient matrix is reported. 
	
	2) ${\rm SC_{BSV}}$ \cite{ng2002OnSpectralClustering} : Best Single-View clustering results achieved by Spectral Clustering. We perform spectral clustering on each single view, and report the best clustering performance. Both ${\rm LRR_{BSV}}$ and ${\rm SC_{BSV}}$ are employed here for demonstrating that the proposed multi-view clustering method can achieve the better clustering performance that sing-view clustering method.
	
	2) KerAdd \cite{cortes2009learning}: Kernel Addition tries to combine information of multiple views by constructing a single kernel matrix which is calculated by averaging kernel matrices of all views.
	
	3) Co-reg \cite{kumar2011Co-regMultiviewSpectralClustering}: It is a co-regularized multi-view spectral clustering which gets multi-view clustering results by pursuing the consistent properties of multi-view data.
	
	4) RMSC \cite{xia2014RMSC} : Robust Multi-view Spectral Clustering is an effective multi-view clustering approach, which recovers a shared low-rank transition probability matrix for clustering.
	
	5) AMGL \cite{IJCAI16AMGLnie2016parameter}: Auto-Weighted Multiple Graph Learning, which learn parameters of weights automatically for multiple graphs and can be employed for multi-view clustering.
	
	6) LMSC \cite{zhang2017LMSC}: Latent Multi-view Subspace Clustering, which explores the consensus information and complementary information by seeking a latent representation of all views for clustering.
	
	7) MLRSSC \cite{brbic2018MLRSSC}: Multi-view Low-Rank Sparse Subspace Clustering. Both the low-rank and sparsity constraints are employed to get an affinity matrix for mutli-view clustering. Linear kernel MLRSSC algorithm is employed here for comparison.
	
	8) GMC \cite{GMCTKDE2019}: Graph-based Multi-view Clustering. It learns the data graph of different views and achieve multi-view clustering results by fussing them into a integrated graph matrix.
	
	{\color{black}{For a fair comparison, we tune parameters of the comparing methods to seek the optimal performance by following the recommended setting in the original works.}} Furthermore, five widely used metrics \cite{manning2010introduction} for evaluation is employed here, including NMI (Normalized Mutual Information), ACC (ACCuracy), F-score, AVG (AVGent), and P (Precision). Excepting for AVG, higher values of all metrics indicate the better clustering performance for all metrics. For each dataset, 30 test runs with random initialization were conducted. Experimental results are reported in form of the average value and the standard deviation. 
	
	\subsection{Experimental results}
	Comprehensive experiments conducted on the benchmark datasets illustrate that the proposed CBF-MSC can achieve promising and competitive clustering results. That is to say under the assumption that all views have the coefficient matrices with the same clustering property, consensus information and complementary information are explored effectively during multi-view clustering.  
	
	The results of comparison experiments are presented in Table \ref{ClusteringResults_MSRCV1}-\ref{ClusteringResults_Caltech101}, and bold values indicate the best clustering performance. Overall speaking, for multi-view data, multi-view clustering methods can obtains more promising clustering results than single-view clustering method. For example, compared to $\rm{SC}_{BSV}$, the proposed CBF-MSC achieve about $23.35\%$ and $23.07\%$ improvements on the MSRCV1 datraset with repsect to the metrics of NMI and ACC, respectively.
	
	Moreover, as a whole, the proposed method outperforms all the competed methods on all nine benchmark datasets, since the CBF-MSC gets clustering results under a more suitable assumption and explores the underlying clustering structures of multi-view data more effectively. For example, on 3-sources dataset, the CBF-MSC gains an increase of $7.27\%$ and $9.82\%$ compared with the second best clustering results in the metrics of NMI and ACC, respectively. About $6.0\%$ and $5.72\%$ improvements are achieved on the Yale Face dataset in the metrics of NMI and ACC as well. Additionally, on the Movie 617 dataset, $2.04\%$ and $2.84\%$ increses are achieved with respect to the metric of NMI and ACC. It is notable that the proposed method is pretty robust among different benchmark datasets and clustering results of the CBF-MSC are promising as well.
	
	\begin{table*}
		\setlength{\tabcolsep}{4.5mm}
		\begin{center}
			\caption{Clustering Results on the MSRCV1 dataset.}\label{ClusteringResults_MSRCV1}
			\begin{tabular}{c|ccccc}
				\hline
				\bf Method & \bf NMI & \bf ACC & \bf F-score & \bf AVG & \bf P\\
				\hline
				${\rm LRR_{BSV}}$ & 0.5698(0.0068) & 0.6740(0.0081) & 0.5363(0.0082) & 1.2135(0.0204) & 0.5291(0.0097)\\
				\hline
				${\rm SC_{BSV}}$ & 0.6046(0.0184) & 0.6830(0.0257) & 0.5716(0.0211) & 1.1174(0.0526) & 0.5628(0.0224)\\
				\hline
				KerAdd & 0.6178(0.0092) & 0.7125(0.0134) & 0.5978(0.0103) & 1.0866(0.0270) & 0.5812(0.0117)\\
				\hline
				Co-reg & 0.6572(0.0103) & 0.7653(0.0186) & 0.6441(0.0140) & 0.9713(0.0301) & 0.6328(0.0156)\\
				\hline
				RMSC & 0.6732(0.0070) & 0.7886(0.0154) & 0.6661(0.0090) & 0.9253(0.0214) & 0.6559(0.0117)\\
				\hline
				AMGL & 0.7357(0.0281) & 0.7171(0.0847) & 0.6445(0.0601) & 0.8236(0.1205) & 0.5686(0.0889)\\
				\hline
				LMSC & 0.6149(0.0622) & 0.6948(0.0734) & 0.5909(0.0699) & 1.0965(0.1744) & 0.5728(0.0699)\\
				\hline
				MLRSSC &0.6709(0.0352) & 0.7775(0.0497) & 0.6524(0.0480) & 0.9290(0.0983) & 0.6452(0.0471)\\
				\hline
				GMC & 0.8200(0.0000) & 0.8952(0.0000) & 0.7997(0.0000) & 0.5155(0.0000) & 0.7856(0.0000)\\
				\hline
				CBF-MSC & \bf 0.8381(0.0036) & \bf 0.9137(0.0016) & \bf 0.8363(0.0030) & \bf 0.4595(0.0100) & \bf 0.8289(0.0028)\\
				\hline
			\end{tabular}
		\end{center}
	\end{table*}
	
	\begin{table*}
		\setlength{\tabcolsep}{4.5mm}
		\begin{center}
			\caption{Clustering Results on the 3-sources dataset.}\label{ClusteringResults_3sources}
			\begin{tabular}{c|ccccc}
				\hline
				\bf Method & \bf NMI & \bf ACC & \bf F-score & \bf AVG & \bf P\\
				\hline
				${\rm LRR_{BSV}}$ & 0.6348(0.0078) & 0.6783(0.0136) & 0.6158(0.0185) & 0.7958(0.0150) & 0.6736(0.0148)\\
				\hline
				${\rm SC_{BSV}}$ & 0.4688(0.0078) & 0.5567(0.0129) & 0.4746(0.0106) & 1.1880(0.0206) & 0.5175(0.0138)\\
				\hline
				KerAdd & 0.4623(0.0099) & 0.5437(0.0119) & 0.4689(0.0096) & 1.1972(0.0231) & 0.5218(0.0117)\\
				\hline
				Co-reg & 0.4790(0.0060) & 0.5067(0.0075) & 0.4352(0.0038) & 1.1693(0.0166) & 0.4629(0.0079)\\
				\hline
				RMSC & 0.5109(0.0100) & 0.5379(0.0108) & 0.4669(0.0097) & 1.0946(0.0254) & 0.4970(0.0136)\\
				\hline
				AMGL & 0.5865(0.0510) & 0.6726(0.0394) & 0.5895(0.0414) & 1.0841(0.1344) & 0.4865(0.0592)\\
				\hline
				LMSC & 0.6748(0.0195) &0.7059(0.0198) & 0.6451(0.0177) & 0.6827(0.0496) & \bf 0.7314(0.0237)\\
				\hline
				MLRSSC & 0.5919(0.0025) & 0.6686(0.0000) & 0.6353(0.0011) & 0.9378(0.0070) & 0.6410(0.0018)\\
				\hline
				GMC & 0.6216(0.0000) & 0.6923(0.0000) & 0.6047(0.0000) & 1.0375(0.0000) & 0.4844(0.0000)\\
				\hline
				CBF-MSC & \bf 0.7476(0.0048) & \bf 0.8041(0.0018) & \bf 0.7727(0.0019) & \bf 0.6322(0.0063) & 0.7181(0.0021)\\
				\hline
			\end{tabular}
		\end{center}
	\end{table*}
	
	\begin{table*}
		\setlength{\tabcolsep}{4.5mm}
		\begin{center}
			\caption{Clustering Results on the BBC dataset.}\label{ClusteringResults_BBC}
			\begin{tabular}{c|ccccc}
				\hline
				\bf Method & \bf NMI & \bf ACC & \bf F-score & \bf AVG & \bf P\\
				\hline
				${\rm LRR_{BSV}}$ & 0.5314(0.0009) & 0.7332(0.0009) & 0.5806(0.0010) & 1.0170(0.0020) & 0.5940(0.0012)\\
				\hline
				${\rm SC_{BSV}}$ & 0.2930(0.0025) & 0.4537(0.0056) & 0.4019(0.0042) & 1.5820(0.0094) & 0.6854(0.0068)\\
				\hline
				KerAdd & 0.4716(0.0051) & 0.6591(0.0144) & 0.5636(0.0105) & 1.1717(0.0112) & 0.5475(0.0090)\\
				\hline
				Co-reg & 0.4183(0.0077) & 0.6160(0.0129) & 0.5070(0.0103) & 1.2824(0.0157) & 0.4964(0.0091)\\
				\hline
				RMSC & 0.5412(0.0076) & 0.6861(0.0225) & 0.5697(0.0084) & 0.9936(0.0153) & 0.5838(0.0061)\\
				\hline
				AMGL & 0.5185(0.0725) & 0.6261(0.0698) & 0.6050(0.0527) & 1.2376(0.1686) & 0.4776(0.0588)\\
				\hline
				LMSC & 0.5594(0.0409) & 0.7394(0.0671) & 0.6291(0.0491) & 0.9455(0.0921) & 0.6600(0.0525)\\
				\hline
				MLRSSC & 0.6935(0.0004) & 0.8556(0.0004) & 0.7897(0.0003) & \bf 0.6681(0.0010) & \bf 0.7966(0.0005)\\
				\hline
				GMC & 0.5628(0.0000) & 0.6934(0.0000) & 0.6333(0.0000) & 1.1288(0.0000) & 0.5012(0.0000)\\
				\hline
				CBF-MSC & \bf 0.6957(0.0012) & \bf 0.8700(0.0007) & \bf 0.7969(0.0009) &  0.6709(0.0030) &  0.7913(0.0015)\\
				\hline
			\end{tabular}
		\end{center}
	\end{table*}
	
	\begin{table*}
		\setlength{\tabcolsep}{4.5mm}
		\begin{center}
			\caption{Clustering Results on the NGs dataset.}\label{ClusteringResults_NGs}
			\begin{tabular}{c|ccccc}
				\hline
				\bf Method & \bf NMI & \bf ACC & \bf F-score & \bf AVG & \bf P\\
				\hline
				${\rm LRR_{BSV}}$ & 0.3402(0.0201) & 0.4213(0.0184) & 0.3911(0.0056) & 1.7056(0.0461) & 0.2688(0.0065)\\
				\hline
				${\rm SC_{BSV}}$ & 0.0163(0.0005) & 0.2044(0.0003) & 0.3302(0.0000) & 2.3023(0.0006) & 0.1984(0.0000)\\
				\hline
				KerAdd & 0.1284(0.0013) & 0.3571(0.0017) & 0.2868(0.0019) & 2.0334(0.0033) & 0.2635(0.0010)\\
				\hline
				Co-reg & 0.1815(0.0052) & 0.2999(0.0038) & 0.3391(0.0016) & 2.0142(0.0107) & 0.2217(0.0016)\\
				\hline
				RMSC & 0.1580(0.0099) & 0.3700(0.0081) & 0.3070(0.0058) & 1.9755(0.0236) & 0.2664(0.0074)\\
				\hline
				AMGL & 0.8987(0.0464) & 0.9393(0.0903) & 0.9212(0.0709) & 0.2473(0.1385) & 0.9088(0.1024)\\
				\hline
				LMSC & 0.9052(0.0075) & 0.9705(0.0026) & 0.9417(0.0050) & 0.2203(0.0173) & 0.9415(0.0051)\\
				\hline
				MLRSSC & 0.8860(0.0000) & 0.9620(0.0000) & 0.9255(0.0000) & 0.2651(0.0000) &  0.9252(0.0000)\\
				\hline
				GMC & 0.9392(0.0000) & 0.9820(0.0000) & 0.9643(0.0000) & 0.1413(0.0000) & 0.9642(0.0000)\\
				\hline
				CBF-MSC & \bf 0.9478(0.0000) & \bf 0.9840(0.0000) & \bf 0.9682(0.0000) &  \bf 0.1212(0.0000) &  \bf0.9682(0.0000)\\
				\hline
			\end{tabular}
		\end{center}
	\end{table*}
	
	\begin{table*}
		\setlength{\tabcolsep}{4.5mm}
		\begin{center}
			\caption{Clustering Results on the Yale Face dataset.}\label{ClusteringResults_YaleFace}
			\begin{tabular}{c|ccccc}
				\hline
				\bf Method & \bf NMI & \bf ACC & \bf F-score & \bf AVG & \bf P\\
				\hline
				${\rm LRR_{BSV}}$ & 0.7134(0.0098) & 0.7034(0.0125) & 0.5561(0.0159) & 1.1328(0.0390) & 0.5404(0.0176)\\
				\hline
				${\rm SC_{BSV}}$ & 0.6429(0.0087) & 0.6031(0.0137) & 0.4556(0.0117) & 1.4182(0.0339) & 0.4335(0.0123)\\
				\hline
				KerAdd & 0.6692(0.0085) & 0.6133(0.0136) & 0.4922(0.0116) & 1.3181(0.0346) & 0.4665(0.0127)\\
				\hline
				Co-reg & 0.6178(0.0098) & 0.5686(0.0128) & 0.4251(0.0127) & 1.5138(0.0389) & 0.4061(0.0131)\\
				\hline
				RMSC & 0.6812(0.0089) & 0.6283(0.0146) & 0.5059(0.0119) & 1.2692(0.0365) & 0.4819(0.0137)\\
				\hline
				AMGL & 0.6438(0.0192) & 0.6046(0.0399) & 0.3986(0.0323) & 1.4710(0.0919) & 0.3378(0.0431)\\
				\hline
				LMSC & 0.7011(0.0096) & 0.6691(0.0095) & 0.5031(0.0151) & 1.2062(0.0391) & 0.4638(0.0175)\\
				\hline
				MLRSSC & 0.7005(0.0311) & 0.6733(0.0384) & 0.5399(0.0377) & 1.1847(0.1206) & 0.5230(0.0378)\\
				\hline
				GMC & 0.6892(0.0000) & 0.6545(0.0000) & 0.4801(0.0000) & 1.2753(0.0000) & 0.4188(0.0000)\\
				\hline
				CBF-MSC & \bf 0.7734(0.0147) & \bf 0.7606(0.0309) & \bf 0.6163(0.0246) &  \bf 0.9110(0.0598) &  \bf 0.5857(0.0278)\\
				\hline
			\end{tabular}
		\end{center}
	\end{table*}
	
	\begin{table*}
		\setlength{\tabcolsep}{4.5mm}
		\begin{center}
			\caption{Clustering Results on the Movie 617 dataset.}\label{ClusteringResults_M617}
			\begin{tabular}{c|ccccc}
				\hline
				\bf Method & \bf NMI & \bf ACC & \bf F-score & \bf AVG & \bf P\\
				\hline
				${\rm LRR_{BSV}}$ & 0.2690(0.0063) & 0.2767(0.0093) & 0.1566(0.0040) & 2.9462(0.0250) & 0.1528(0.0042)\\
				\hline
				${\rm SC_{BSV}}$ & 0.2600(0.0026) & 0.2567(0.0040) & 0.1473(0.0020) & 2.9931(0.0104) & 0.1365(0.0018)\\
				\hline
				KerAdd & 0.2925(0.0026) & 0.2912(0.0045) & 0.1764(0.0035) & 2.8606(0.0109) & 0.1667(0.0038)\\
				\hline
				Co-reg &  0.2456(0.0019) & 0.2409(0.0022) & 0.1389(0.0017) & 3.0446(0.0075) & 0.1324(0.0016)\\
				\hline
				RMSC & 0.2969(0.0023) & 0.2986(0.0043) & 0.1819(0.0024) & 2.8498(0.0095) & 0.1674(0.0024)\\
				\hline
				AMGL & 0.2607(0.0088) & 0.2563(0.0124) & 0.1461(0.0055) & 3.1105(0.0387) & 0.0971(0.0063)\\
				\hline
				LMSC & 0.2796(0.0096) & 0.2694(0.0133) & 0.1601(0.0088) & 2.9129(0.0388) & 0.1512(0.0092)\\
				\hline
				MLRSSC & 0.2975(0.0061) & 0.2887(0.0111) & 0.1766(0.0068) & 2.8481(0.0216) & 0.1619(0.0064)\\
				\hline
				GMC & 0.2334(0.0000) & 0.18634(0.0000) & 0.1242(0.0000) & 3.3795(0.0000) & 0.0682(0.0000)\\
				\hline
				CBF-MSC & \bf 0.3179(0.0079) & \bf0.3171(0.0070) & \bf 0.1947(0.0062) &  \bf 2.7520(0.0337) &  \bf 0.1893(0.0077)\\
				\hline
			\end{tabular}
		\end{center}
	\end{table*}
	
	\begin{table*}
		\setlength{\tabcolsep}{4.5mm}
		\begin{center}
			\caption{{\color{black}{Clustering Results on the BBCSport dataset.}}}\label{ClusteringResults_BBCSport}
			\begin{tabular}{c|ccccc}
				\hline
				\bf {\color{black}{Method}} & \bf {\color{black}{NMI}} & \bf {\color{black}{ACC}} & \bf {\color{black}{F-score}} & \bf {\color{black}{AVG}} & \bf {\color{black}{P}}\\
				\hline
				{\color{black}{${\rm LRR_{BSV}}$}} & {\color{black}{0.6996(0.0000)}} & {\color{black}{0.7970(0.0015)}} & {\color{black}{0.7612(0.0001)}} & {\color{black}{0.7269(0.0006)}} & {\color{black}{0.6890(0.0001)}}\\
				\hline
				{\color{black}{${\rm SC_{BSV}}$}} & {\color{black}{0.7182(0.0054)}} & {\color{black}{0.8456(0.0099)}} & {\color{black}{0.7671(0.0067)}} & {\color{black}{0.6070(0.0118)}} & {\color{black}{0.7853(0.0063)}}\\
				\hline
				{\color{black}{KerAdd}} & {\color{black}{0.6170(0.0074)}} & {\color{black}{0.6170(0.0070)}} & {\color{black}{0.6696(0.0061)}} & {\color{black}{0.8471(0.0185)}} & {\color{black}{0.6638(0.0078)}}\\
				\hline
				{\color{black}{Co-reg}} & {\color{black}{0.7185(0.0031)}} & {\color{black}{0.8465(0.0050)}} & {\color{black}{0.7674(0.0041)}} & {\color{black}{0.6062(0.0066)}} & {\color{black}{0.7859(0.0033)}}\\
				\hline
				{\color{black}{RMSC}} & {\color{black}{0.8124(0.0074)}} & {\color{black}{0.8562(0.0198)}} & {\color{black}{0.8514(0.0132)}} & {\color{black}{0.4159(0.0149)}} & {\color{black}{0.8566(0.0105)}}\\
				\hline
				{\color{black}{AMGL}} & {\color{black}{0.8640(0.0681)}} & {\color{black}{0.9189(0.0870)}} & {\color{black}{0.9008(0.0868)}} & {\color{black}{0.3305(0.1858)}} & {\color{black}{0.8708(0.1188)}}\\
				\hline
				{\color{black}{LMSC}} & {\color{black}{0.8393(0.0043)}} & {\color{black}{0.9180(0.0031)}} & {\color{black}{0.8996(0.0033)}} & {\color{black}{0.3608(0.0094)}} & {\color{black}{0.8938(0.0036)}}\\
				\hline
				{\color{black}{MLRSSC}} & {\color{black}{0.8855(0.0000)}} & \bf {\color{black}{0.9651(0.0000)}} & \bf {\color{black}{0.9296(0.0000)}} & {\color{black}{0.2437(0.0000)}} & \bf {\color{black}{0.9384(0.0000)}}\\
				\hline
				{\color{black}{GMC}} & {\color{black}{0.7954(0.0000)}} & {\color{black}{0.7390(0.0000)}} & {\color{black}{0.7207(0.0000)}} & {\color{black}{0.6450(0.0000)}} & {\color{black}{0.5728(0.0000)}}\\
				\hline
				{\color{black}{CBF-MSC}} & \bf {\color{black}{0.8911(0.0000)}} & \bf {\color{black}{0.9651(0.0000)}} & {\color{black}{0.9265(0.0000)}} & \bf {\color{black}{0.2310(0.0000)}} & {\color{black}{0.9360(0.0000)}}\\
				\hline
			\end{tabular}
		\end{center}
	\end{table*}
	
	\begin{table*}
		\setlength{\tabcolsep}{4.5mm}
		\begin{center}
			\caption{{\color{black}{Clustering Results on the UCI digits dataset.}}}\label{ClusteringResults_UCI}
			\begin{tabular}{c|ccccc}
				\hline
				\bf {\color{black}{Method}} & \bf {\color{black}{NMI}} & \bf {\color{black}{ACC}} & \bf {\color{black}{F-score}} & \bf {\color{black}{AVG}} & \bf {\color{black}{P}}\\
				\hline
				{\color{black}{${\rm LRR_{BSV}}$}} & {\color{black}{0.6065(0.0008)}} & {\color{black}{0.6202(0.0007)}} & {\color{black}{0.5352(0.0005)}} & {\color{black}{1.3144(0.0027)}} & {\color{black}{0.52623(0.0005)}}\\
				\hline
				{\color{black}{${\rm SC_{BSV}}$}} & {\color{black}{0.6378(0.0047)}} & {\color{black}{0.6771(0.0114)}} & {\color{black}{0.5739(0.0067)}} & {\color{black}{1.2114(0.0165)}} & {\color{black}{0.5641(0.0085)}}\\
				\hline
				{\color{black}{KerAdd}} & {\color{black}{0.7849(0.0034)}} & {\color{black}{0.8343(0.0135)}} & {\color{black}{0.7524(0.0075)}} & {\color{black}{0.7255(0.0139)}} & {\color{black}{0.7386(0.0116)}}\\
				\hline
				{\color{black}{Co-reg}} & {\color{black}{0.7169(0.0058)}} & {\color{black}{0.7821(0.0143)}} & {\color{black}{0.6843(0.0092)}} & {\color{black}{0.9492(0.0207)}} & {\color{black}{0.6732(0.0112)}}\\
				\hline
				{\color{black}{RMSC}} & {\color{black}{0.8259(0.0085)}} & {\color{black}{0.8550(0.0161)}} & {\color{black}{0.8014(0.0146)}} & {\color{black}{0.5897(0.0304)}} & {\color{black}{0.7876(0.0179)}}\\
				\hline
				{\color{black}{AMGL}} & {\color{black}{0.8038(0.0175)}} & {\color{black}{0.7822(0.0447)}} & {\color{black}{0.7359(0.0370)}} & {\color{black}{0.6886(0.0712)}} & {\color{black}{0.6965(0.0521)}}\\
				\hline
				{\color{black}{LMSC}} & {\color{black}{0.7826(0.0256)}} & {\color{black}{0.8578(0.0298)}} & {\color{black}{0.7618(0.0368)}} & {\color{black}{0.7256(0.0856)}} & {\color{black}{0.7571(0.0379)}}\\
				\hline
				{\color{black}{MLRSSC}} & \bf {\color{black}{0.8365(0.0005)}} & {\color{black}{0.9137(0.0002)}} & \bf {\color{black}{0.8400(0.0004)}} & \bf {\color{black}{0.5456(0.0017)}} & {\color{black}{0.8366(0.0004)}}\\
				\hline
				{\color{black}{GMC}} & {\color{black}{0.8153(0.0000)}} & {\color{black}{0.7355(0.0000)}} & {\color{black}{0.7134(0.0000)}} & {\color{black}{0.6788(0.0000)}} & {\color{black}{0.6443(0.0000)}}\\
				\hline
				{\color{black}{CBF-MSC}} & {\color{black}{0.8357(0.0001)}} & \bf {\color{black}{0.9145(0.0001)}} & \bf {\color{black}{0.8400(0.0001)}} & {\color{black}{0.5471(0.0004)}} & \bf {\color{black}{0.8380(0.0001)}}\\
				\hline
			\end{tabular}
		\end{center}
	\end{table*}
	
	\begin{table*}
		\setlength{\tabcolsep}{4.5mm}
		\begin{center}
			\caption{{\color{black}{Clustering Results on the Caltech101 dataset.}}}\label{ClusteringResults_Caltech101}
			\begin{tabular}{c|ccccc}
				\hline
				\bf {\color{black}{Method}} & \bf {\color{black}{NMI}} & \bf {\color{black}{ACC}} & \bf {\color{black}{F-score}} & \bf {\color{black}{AVG}} & \bf {\color{black}{P}}\\
				\hline
				{\color{black}{${\rm LRR_{BSV}}$}} & {\color{black}{0.4395(0.0018)}} & {\color{black}{0.5284(0.0009)}} & {\color{black}{0.4744(0.0014)}} & {\color{black}{1.5236(0.0042)}} & {\color{black}{0.4509(0.0007)}}\\
				\hline
				{\color{black}{${\rm SC_{BSV}}$}} & {\color{black}{0.4429(0.0052)}} & {\color{black}{0.5491(0.0094)}} & {\color{black}{0.5001(0.0060)}} & {\color{black}{1.5226(0.0156)}} & {\color{black}{0.4688(0.0070)}}\\
				\hline
				{\color{black}{KerAdd}} & {\color{black}{0.3615(0.0053)}} & {\color{black}{0.4931(0.0067)}} & {\color{black}{0.4194(0.0072)}} & {\color{black}{1.7361(0.0111)}} & {\color{black}{0.3932(0.0031)}}\\
				\hline
				{\color{black}{Co-reg}} & {\color{black}{0.3826(0.0034)}} & {\color{black}{0.5085(0.0072)}} & {\color{black}{0.4302(0.0062)}} & {\color{black}{1.6697(0.0090)}} & {\color{black}{0.4172(0.0041)}}\\
				\hline
				{\color{black}{RMSC}} & {\color{black}{0.3736(0.0058)}} & {\color{black}{0.5100(0.0098)}} & {\color{black}{0.4172(0.0081)}} & {\color{black}{1.6733(0.0147)}} & {\color{black}{0.4359(0.0053)}}\\
				\hline
				{\color{black}{AMGL}} & {\color{black}{0.3927(0.0199)}} & {\color{black}{0.4720(0.0154)}} & {\color{black}{0.3651(0.0072)}} & {\color{black}{1.7756(0.0539)}} & {\color{black}{0.2644(0.0171)}}\\
				\hline
				{\color{black}{LMSC}} & {\color{black}{0.4372(0.0179)}} & {\color{black}{0.5182(0.0074)}} & {\color{black}{0.4517(0.0082)}} & {\color{black}{1.5276(0.0592)}} & {\color{black}{0.4344(0.0227)}}\\
				\hline
				{\color{black}{MLRSSC}} & {\color{black}{0.4817(0.0007)}} & {\color{black}{0.5312(0.0010)}} & {\color{black}{0.4694(0.0012)}} & {\color{black}{1.4181(0.0016)}} & {\color{black}{0.4326(0.0006)}}\\
				\hline
				{\color{black}{GMC}} & {\color{black}{0.4240(0.0000)}} & {\color{black}{0.4785(0.0000)}} & {\color{black}{0.3556(0.0000)}} & {\color{black}{1.7512(0.0000)}} & {\color{black}{0.2425(0.0000)}}\\
				\hline
				{\color{black}{CBF-MSC}} & \bf {\color{black}{0.4870(0.0009)}} & \bf {\color{black}{0.6165(0.0008)}} & \bf {\color{black}{0.5068(0.0012)}} & \bf {\color{black}{1.3655(0.0025)}} & \bf {\color{black}{0.5281(0.0013)}}\\
				\hline
			\end{tabular}
		\end{center}
	\end{table*}
	
	\subsection{Convergence Analysis}
	Convergence properties of the proposed algorithm are analyzed in this section as well. To be specific, three convergence conditions are defined as follows:
	\begin{equation}
	\begin{array}{l}
	X_{\rm{Convergence }}= \frac{1}{v}\sum\limits_{i = 1}^v {{{\left\| {{X^{(i)}} - {X^{(i)}}{Z^{(i)}} - {E^{(i)}}} \right\|}_\infty }} \\
	Z_{\rm{Convergence}}=\frac{1}{v}\sum\limits_{i = 1}^v {{{\left\| {{Z^{(i)}} - {U^{(i)}}V} \right\|}_\infty }} \\
	V_{\rm{Convergence}}={\left\| {V - L} \right\|_\infty }
	\end{array}
	\end{equation}
	
	As shown in Fig. \ref{Conv_ana}, three convergence curves corresponding to $X_{\rm{Convergence}}$, $Z_{\rm{Convergence}}$, and $V_{\rm{Convergence}}$ versus the iteration numbers are presented and it is clear that the proposed algorithm can achieve convergence within 30 iterations on all benchmark datasets. Although the solid proof of the convergence is difficult to deliver, experimental results verifies the effectiveness and good convergence of the proposed algorithm.
	
	\begin{figure*}
		\begin{center}
			\includegraphics[width=1\linewidth]{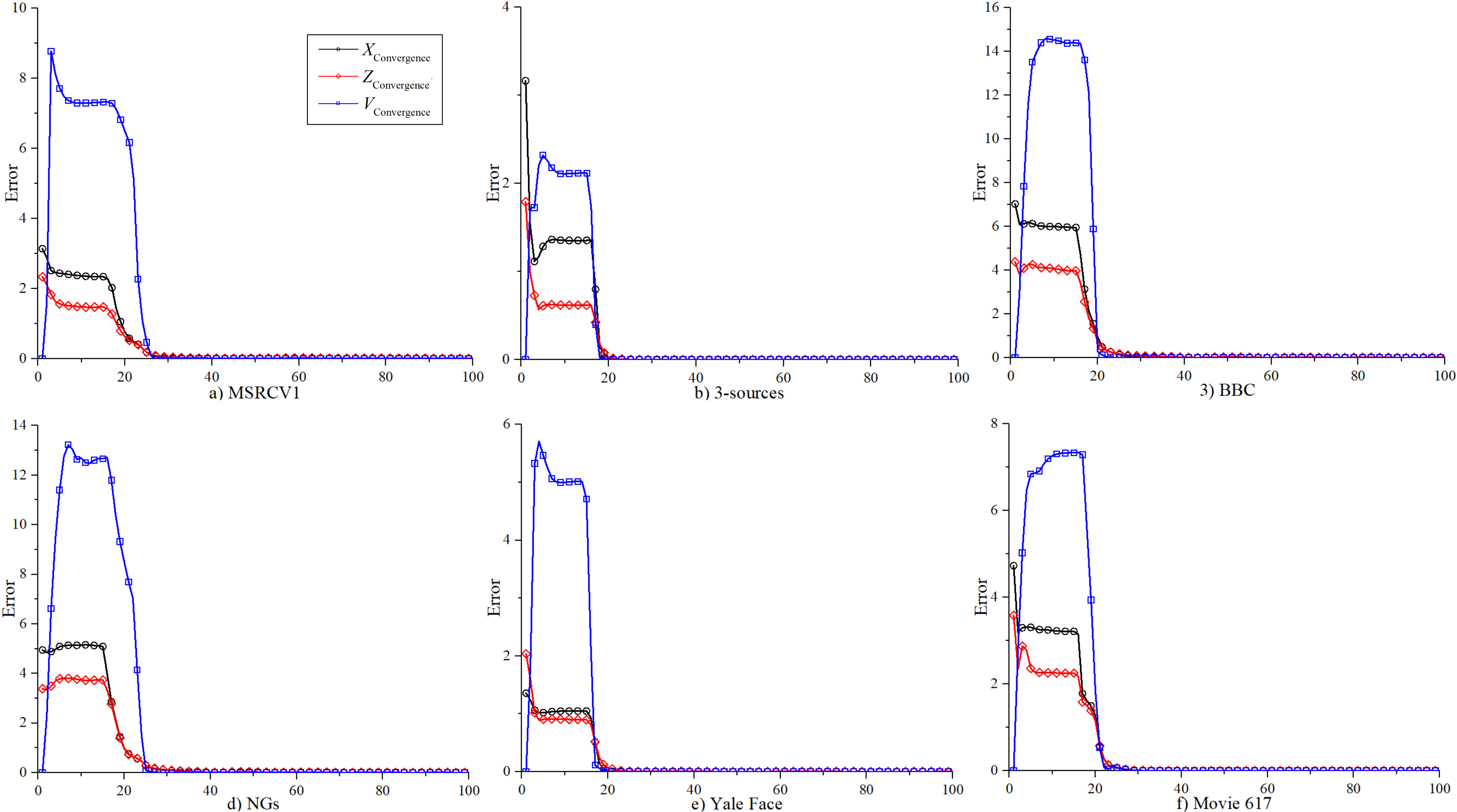}
		\end{center}
		\caption{Convergence curves of the propose method conducted on six benchmark datasets. Empirically, these curves demonstrate the stable and strong convergence of our CBF-MSC.}
		\label{Conv_ana}
	\end{figure*}
	
	\subsection{Parameters Sensitivity}
	Parameter sensitivity of the proposed method will be disscused in this section. Two parameters, which affect the clusteirng performance, should be tuned, i.e., $\lambda$ and $k$. $\lambda$ is a tradeoff parameter, shown in Equation (\ref{CBFMSC_ObjectiveFunction}). Generally speaking, the prior information about the data, such as the noise level, is important for the tuning of $\lambda$. $k$ is a dimensional parameter introduced in Equation (\ref{BilinearFactorization_Introducted}). To be specific, given a dataset with $n$ data samples drawn from $c$ subspaces, we set $k$ belong to $[c,n]$, and for simplicity, $k$ is chosen from the set $\left\{ {c,2c,3c,...} \right\}$ and let $k<n$ simultaneously.
	
	Taking MSRCV1 for example, Fig. \ref{lambda_tuning} and Fig. \ref{k_tuning} present the clustering results with different values of $\lambda$ and $k$, respectively. It can be observed that $\lambda=100$ achieves the promising clustering results. As for parameter $k$, as a whole, it is clear that the clustering performance is robust to different values of $k$ and $k=35$ obtains the best clustering results.
	
	\begin{figure}
		\begin{center}
			\includegraphics[width=1\linewidth]{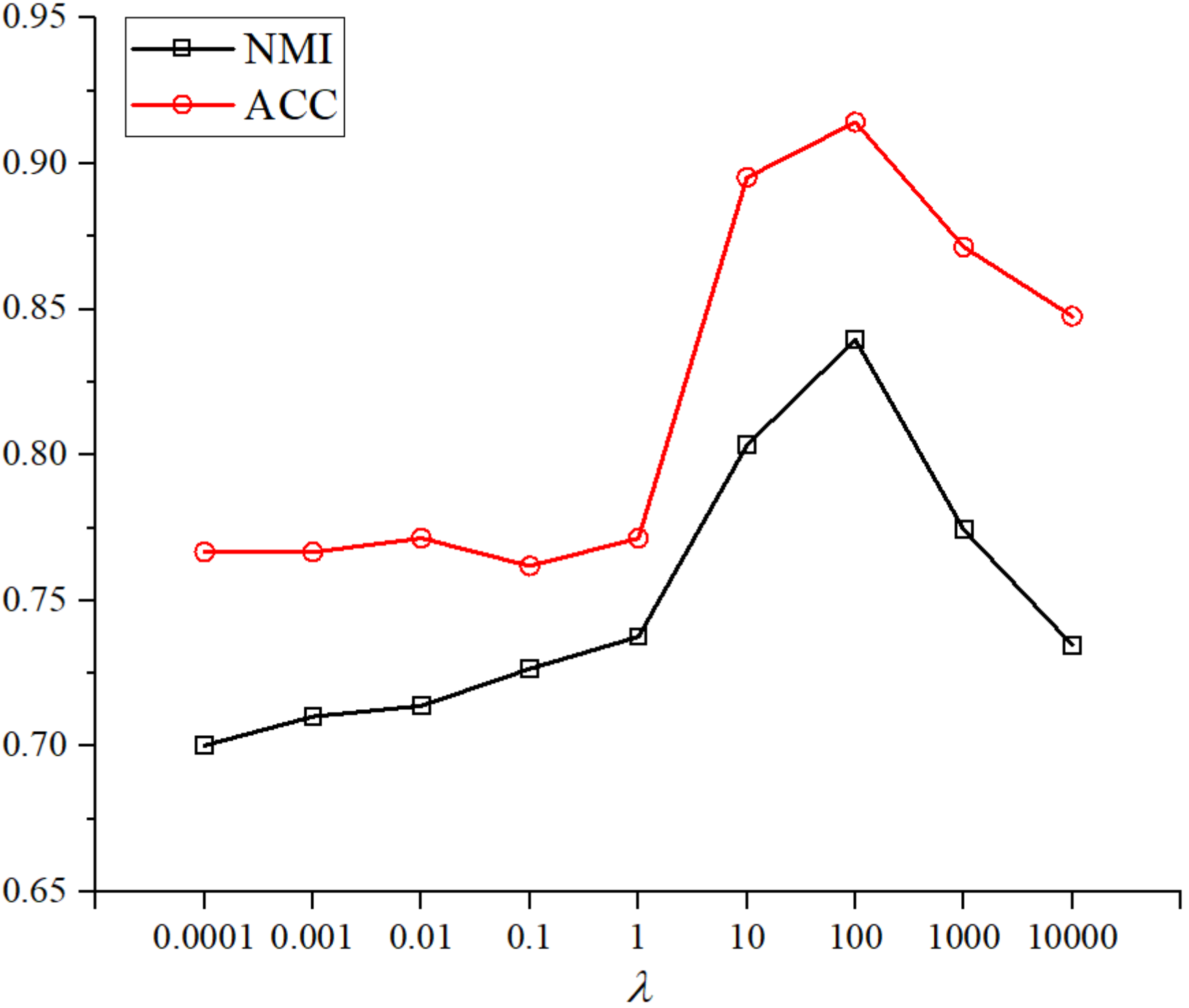}
		\end{center}
		\caption{Clustering results with $k=35$ and different values of $\lambda$ on MSRCV1.}
		\label{lambda_tuning}
	\end{figure}
	
	\begin{figure}
		\begin{center}
			\includegraphics[width=1\linewidth]{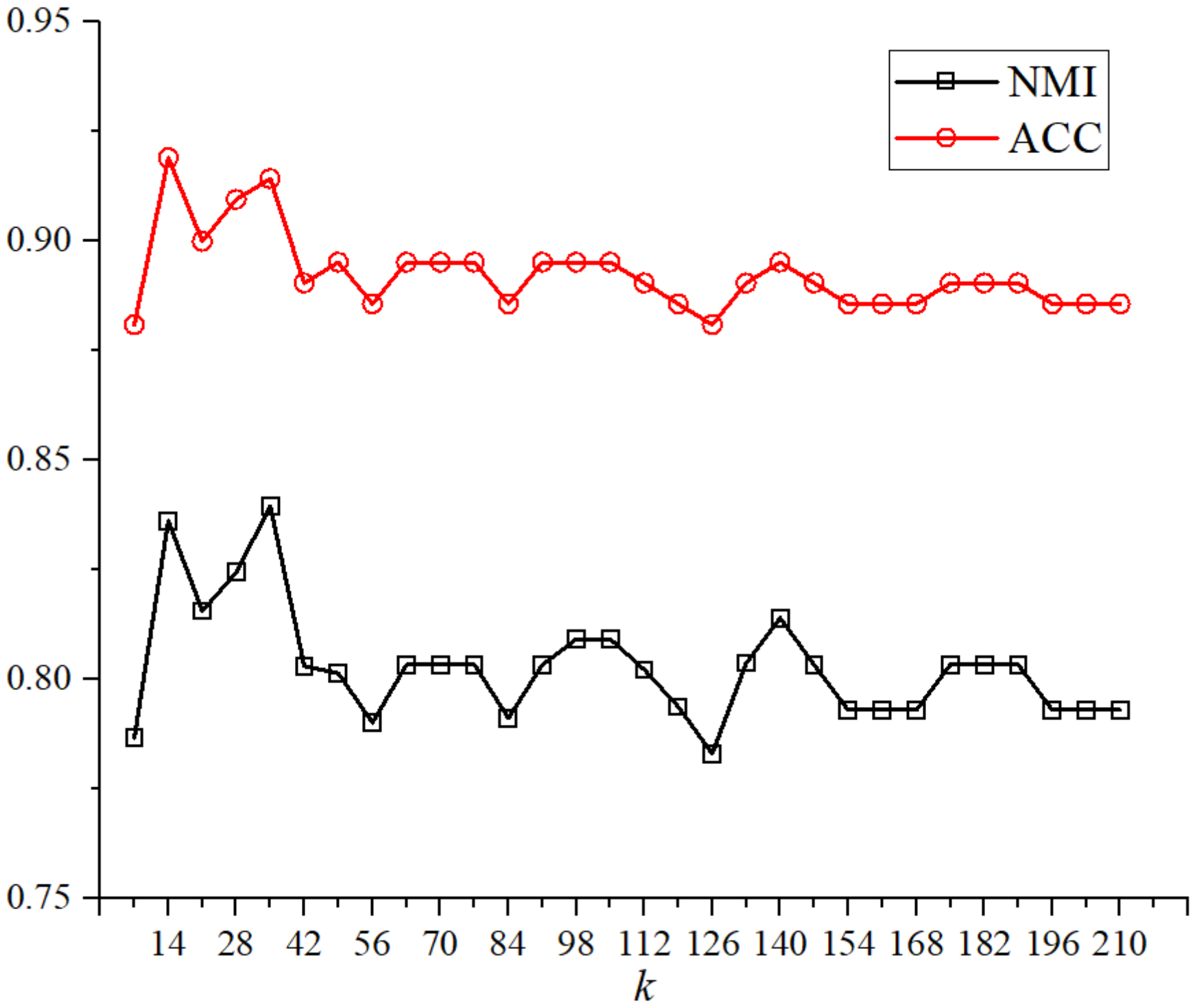}
		\end{center}
		\caption{Clustering results with $\lambda=100$ and different values of $k$ on MSRCV1.}
		\label{k_tuning}
	\end{figure}

	\section{Conclusion}
	\label{Conclusion}
	In this paper, we develop a novel multi-view subspace clustering approach, dubbed CBF-MSC. {\color{black}{Different from most exsiting multi-view clustering methods, the proposed method assumes that the coefficient matrices of different views have the same clustering structure.}} To achieve the assumption, the orthonormality constrained bilinear factorization is introduced on the coefficient matrices of all views. Furthermore, an effective optimization is also developed in this paper to solve the objective function of the corresponding problem. Comprehensive experimental results validate the effectiveness of the proposed method and show that our CBF-MSC outperforms several state-of-the-arts.
	
	{\color{black}{Despite the promising clustering performance can be achieved, the proposed method is time consuming since the SVD decomposition and matrix inversion are involved during the optimization procedure. The improvement for large-scale data will be focused on in further works by leveraging the Hashing technique and the dimensionality reduction strategies.}}
	
	\section*{Acknowledgements}
	
	This work is supported by the National Natural Science Foundation of China under Grant No. 61573273.
	
	\printcredits
	
	\bibliographystyle{cas-model2-names}
	
	\bibliography{CBFMSC}
	
\end{document}